\documentclass[journal]{IEEEtran}

\ifCLASSINFOpdf
  \usepackage[pdftex]{graphicx}
  \graphicspath{{resources_en/}}
  \DeclareGraphicsExtensions{.pdf,.jpeg,.png}
\else
  \usepackage[dvips]{graphicx}
  \graphicspath{{resources_en/}}
  \DeclareGraphicsExtensions{.eps}
\fi

\usepackage{amsmath}
\usepackage{amssymb}
\usepackage{amsthm}
\usepackage{mathtools}
\usepackage{algorithm}
\usepackage{algpseudocode}

\usepackage{subcaption}

\usepackage{array}

\usepackage{tikz}
\usetikzlibrary{arrows.meta, shapes, positioning, shadows, fit, decorations.markings}

\usepackage{url}

\newtheorem{lemma}{Lemma}
\newtheorem{definition}{Definition}
\newtheorem{proposition}{Proposition}

\begin{document}

% *** TITLE AND AUTHORS ***
\title{Local Properties of Neural Networks Through the Lens of Layer-wise Hessians}

\author{Maxim~Bolshim~and~Alexander~Kugaevskikh%
\thanks{M. Bolshim and A. Kugaevskikh are with ITMO University, Saint Petersburg, Russia.
E-mail: maxim.bolshim@yandex.ru, a-kugaevskikh@yandex.ru}%
\thanks{Manuscript received November 6, 2025. This work received no specific grant from any funding agency in the public, commercial, or not-for-profit sectors.}%
\thanks{Author Contributions: M. Bolshim contributed to conceptualization, methodology, software, validation, formal analysis, investigation, data curation, writing of the original draft, and visualization. A. Kugaevskikh contributed to supervision, validation, resources, writing review and editing, and project administration.}%
\thanks{AI-Assisted Tools: Claude LLM was used solely for stylistic editing and formatting of the manuscript text. All intellectual content, scientific methodology, experimental work, analysis, and conclusions are original contributions of the authors.}%
}

% Paper header
\markboth{IEEE Transactions on Neural Networks and Learning Systems}%
{Bolshim \MakeLowercase{\textit{et al.}}: Local Properties of Neural Networks Through Layer-wise Hessians}

% Make title
\maketitle

% *** ABSTRACT ***
\begin{abstract}
This paper introduces a novel methodology for analyzing neural networks through the lens of layer-wise
Hessian matrices. We formalize the concept of the local Hessian for individual functional blocks
(layers) of a neural network and demonstrate its utility for characterizing the geometry of the parameter
space. The spectral properties of local Hessians are shown to provide quantitative insights into phenomena
such as overfitting, underparameterization, and the expressivity of neural architectures. We conduct a
comprehensive empirical study across 111 experiments on 37 datasets, revealing consistent patterns in the
evolution and structure of local Hessians during training. These findings establish a foundation for
principled diagnostics and informed design of neural network architectures based on their local geometric properties.
The proposed approach enables early detection of architectural problems and provides quantitative metrics
for optimization quality assessment. Our analysis demonstrates that spectral characteristics of local Hessians
correlate strongly with generalization capability and can serve as indicators of optimal parameterization.
\end{abstract}

% *** KEYWORDS ***

\IEEEpeerreviewmaketitle

% *** INTRODUCTION ***
\section{Introduction}
\IEEEPARstart{D}{eep} neural networks have demonstrated outstanding results in many fields, including computer vision, natural
language processing, and other machine learning tasks \cite{lecun2015deep, goodfellow2016deep}. However,
despite their practical success, the question of why certain architectures outperform others and how to
formally and systematically improve neural network design remains open. The empirical approach, based on
trial and error, is becoming increasingly costly as model sizes and data volumes grow.

A number of researches have demonstrated that analyzing the curvature of the loss landscape via Hessians and
related spectral tools can shed light on the training dynamics and generalization ability of neural networks
\cite{sagun2017empirical, maheswaranathan2019universality, lee2019wide, arora2019finegrained}. In this work,
we put forward the thesis that local properties of the parameter space of a neural network can provide early
insights into the internal characteristics of the model, even at the initial stages of training
\cite{poole2016transient}. Specifically, we propose to utilize local Hessian matrices-matrices of second
derivatives of the objective function with respect to the parameters of individual layers-to analyze the
geometry of the parameter space.

The concept of the local Hessian enables a formal and quantitative characterization of the geometric
properties of the parameter space in the vicinity of an optimization point. In particular, we demonstrate
that the spectrum of the local Hessian-such as the distribution and structure of its eigenvalues
\cite{ghorbani2019investigation}-is closely related to the functional properties of the corresponding layers
of a neural network.

\subsection{Contributions}

The main contributions of this work are as follows:

\begin{itemize}
  \item The introduction of a mathematically rigorous definition of the local Hessian for functional blocks
    of a neural network and an efficient computational algorithm for its calculation.
  \item A detailed analysis of the spectral properties of local Hessians during neural network training
    across 111 experiments involving various architectures and 37 datasets.
  \item An investigation of the geometric interpretation of the neural network parameter space through the
    lens of local Hessians, revealing connections between spectral properties and model performance.
  \item Practical guidelines for architecture optimization based on spectral analysis of local Hessians,
    including detection of overfitting, underparameterization, and suboptimal parameter distribution.
\end{itemize}

The results obtained not only deepen our theoretical understanding of deep neural networks, but also open new
perspectives for studying their internal dynamics and provide actionable insights for practitioners.

\subsection{Paper Organization}

The remainder of this paper is organized as follows. Section~\ref{sec:framework} establishes the mathematical
framework and formal definitions of neural networks, functional blocks, and local Hessians. Section~\ref{sec:localhessian}
introduces the concept of local layer-wise Hessians, derives their properties, and presents an efficient
computational algorithm. Section~\ref{sec:methodology} describes the experimental methodology, including
dataset selection, architectural variations, and data collection procedures. Section~\ref{sec:results}
presents comprehensive empirical results from 111 experiments, analyzing spectral properties and their
correlation with model performance. Section~\ref{sec:discussion} discusses the implications of our findings
for understanding neural network behavior. Finally, Section~\ref{sec:conclusion} concludes the paper and
outlines directions for future research.

% *** MATHEMATICAL FRAMEWORK ***
\section{Mathematical Framework of Neural Networks}
\label{sec:framework}

This section establishes the formal mathematical foundations required for our analysis of local Hessians
in neural networks. We provide rigorous definitions of neural networks, their functional blocks, and
the intermediate representations that arise during forward propagation.

\subsection{Definition and Structure of a Neural Network}

\begin{definition}
  A neural network $\mathcal{F}: \mathbb{R}^d \rightarrow \mathbb{R}^m$ is a parameterized function with
  parameters $\theta \in \mathbb{R}^P$, mapping input data $x \in \mathbb{R}^d$ to an output space via a
  sequence of functional transformations. We denote by $\mathcal{F}(x; \theta)$ the result of applying the
  network to input $x$ with given parameters $\theta$.
\end{definition}

\begin{definition}
  A functional block (layer) $C_i$ of a neural network $\mathcal{F}$ is defined as a pair of modules $(P_i,
  A_i)$, where:
  \begin{itemize}
    \item $P_i: \mathbb{R}^{d_i} \times \mathbb{R}^{p_i} \rightarrow \mathbb{R}^{q_i}$ is a parameterized
      transformation with parameters $\theta_i \in \mathbb{R}^{p_i}$.
    \item $A_i: \mathbb{R}^{q_i} \rightarrow \mathbb{R}^{q_i}$ is an activation function (potentially the identity).
  \end{itemize}
\end{definition}

\begin{definition}
  A neural network $\mathcal{F}$ can be represented as a composition of $n$ functional blocks:
  \begin{equation}
    \mathcal{F}(x; \theta) = (C_n \circ C_{n-1} \circ \ldots \circ C_1)(x),
  \end{equation}
  where $C_i(z) = A_i(P_i(z; \theta_i))$ for input $z$, and $\theta = \{\theta_1, \theta_2, \ldots,
  \theta_n\}$ is the complete set of network parameters.
\end{definition}

This representation of a neural network enables the analysis of each functional block independently, which is
essential for local analysis of network properties. Decomposing a complex model into simpler components is a
key methodological approach that allows the application of spectral analysis tools to individual components.

\subsection{Intermediate Representations and Activation Functions}

\begin{definition}
  The intermediate representation $z_i$ is defined as the input to block $C_i$:
  \begin{equation}
    z_i =
    \begin{cases}
      x, & \text{if } i = 1 \\
      (C_{i-1} \circ \ldots \circ C_1)(x), & \text{if } i > 1
    \end{cases}
  \end{equation}

  Accordingly, the output of block $C_i$ is denoted as:
  \begin{equation}
    y_i = C_i(z_i) = A_i(P_i(z_i; \theta_i))
  \end{equation}
\end{definition}

Intermediate representations play a crucial role in neural network analysis, as they capture how the input
signal is transformed at each processing stage. Of particular interest is the study of the geometry of these
intermediate representations and their relationship to the parameters of the corresponding layers.

\begin{definition}
  For block $C_i$, we define the local scalar function $S_i: \mathbb{R}^{p_i} \rightarrow \mathbb{R}$ as:
  \begin{equation}
    S_i(\theta_i) = \varphi(A_i(P_i(z_i; \theta_i))),
  \end{equation}
  where $\varphi: \mathbb{R}^{q_i} \rightarrow \mathbb{R}$ is an aggregation function, typically $\varphi(y)
  = \sum_{j=1}^{q_i} y_j$.
\end{definition}

The scalar function provides a means to assess the influence of the parameters of a specific layer on its
output for a fixed input. This function is central for defining the local Hessian in the following section.

\subsection{Typical Implementations in Neural Networks}

In modern neural networks, the following component implementations are common:
\begin{itemize}
  \item $P_i$ - linear transformation $P_i(z_i; \theta_i) = W_i z_i + b_i$, where $\theta_i = \{W_i, b_i\}$
  \item $A_i$ - nonlinear activation function, e.g., ReLU, Sigmoid, or Tanh
  \item $\varphi(y_i) = \sum_{j=1}^{q_i} y_{i,j}$ - summation of all components of the output vector
\end{itemize}

These definitions and notations will be used throughout the remainder of this work to ensure mathematical
rigor and consistency.

% *** LOCAL LAYER-WISE HESSIANS ***
\section{Local Layer-wise Hessians in Neural Networks}
\label{sec:localhessian}

This section introduces the central concept of this work: the local Hessian matrix for individual layers
of a neural network. We provide a formal definition, discuss its geometric interpretation, and present
an efficient computational algorithm suitable for large-scale models.

\subsection{Definition of the Local Hessian}

\begin{definition}
  The local Hessian matrix $H_i \in \mathbb{R}^{p_i \times p_i}$ (hereafter, local Hessian, $LH_i$) for block
  $C_i$ is defined as the matrix of second derivatives of the scalar function $S_i$ with respect to the
  parameters $\theta_i$:
  \begin{equation}
    H_i = \nabla_{\theta_i}^2 S_i(\theta_i) = \left[ \frac{\partial^2 S_i(\theta_i)}{\partial \theta_{i,j}
    \partial \theta_{i,k}} \right]_{j,k=1}^{p_i}
  \end{equation}
\end{definition}

Analysis of this matrix provides insights into:

\begin{itemize}[itemsep=0pt, topsep=0pt]
  \item The degree of nonlinearity of the transformation performed by the layer
  \item Interdependencies among parameters and their influence on the layer output
  \item Geometric properties of the parameter space
  \item Sensitivity of the layer to small parameter perturbations
\end{itemize}

In differential geometry, the Hessian of a function at a point defines a quadratic form that approximates the
curvature of the function's level surface. In the context of neural networks, $LH_i$ characterizes the
curvature of the functional response of a layer in its parameter space~\cite{dangel2019modular}. When
analyzing the surface, the sign and distribution of the eigenvalues of $LH_i$ are crucial, as they determine
the local geometry.

\subsection{Efficient Computation of Local Hessians}

The analysis of neural networks using $LH_i$ is a powerful tool; however, due to the quadratic scaling of the
Hessian matrix size with respect to the number of parameters, its direct computation is often computationally
infeasible for modern architectures. Consequently, various approximation methods for $LH_i$ are widely used
in practice~\cite{carlon2024, hare2024, martens2010, nocedal1980}. This section presents a methodology for
working with $LH_i$ that addresses these computational challenges.

For efficient computation of $LH_i$, we propose an algorithm based on the sequential calculation of the matrix rows.

\begin{lemma}
  The elements of the Hessian matrix $H_i$ can be computed sequentially by rows:
  \begin{equation}
    \begin{aligned}
      g_i &= \nabla_{\theta_i} S_i(\theta_i) = \left[ \frac{\partial S_i}{\partial \theta_{i,j}} \right]_{j=1}^{p_i} \\
      H_i[j,:] &= \nabla_{\theta_i} g_{i,j} = \nabla_{\theta_i} \left( \frac{\partial S_i}{\partial
      \theta_{i,j}} \right)
    \end{aligned}
  \end{equation}
\end{lemma}

\begin{proof}
  By the definition of the Hessian matrix, its element $H_i[j,k]$ is given by:
  \begin{equation}
    H_i[j,k] = \frac{\partial^2 S_i(\theta_i)}{\partial \theta_{i,j} \partial \theta_{i,k}}
  \end{equation}

  Denoting $g_{i,j} = \frac{\partial S_i}{\partial \theta_{i,j}}$, we have
  \begin{equation}
    H_i[j,k] = \frac{\partial g_{i,j}}{\partial \theta_{i,k}}
  \end{equation}

  Thus, the $j$-th row of $H_i$ is the gradient of the $j$-th component of the gradient of $S_i$.
\end{proof}

This approach significantly reduces computational costs when working with large models, as it does not require
storing the entire Hessian matrix of size $p_i \times p_i$ in memory simultaneously.

The proposed method is particularly important for the analysis of modern deep neural networks containing millions
of parameters, since the full $LH_i$ matrix for such models would be prohibitively large.
The local approach not only makes the computations practically feasible, but also enables focused analysis of individual
network components, which is often more informative than global analysis.

\subsection{Algorithm for Computing the Local Hessian}

\begin{algorithm}
  \caption{Computation of Local Hessians in a Neural Network}
  \label{alg:hessian}
  \begin{algorithmic}[1]
    \Require Neural network $\mathcal{F}$, input $x \in \mathbb{R}^d$, aggregation function $\varphi$
    \Ensure Set of local Hessian matrices $\{LH_1, LH_2, \ldots, LH_n\}$

    \State Decompose $\mathcal{F}$ into functional blocks $\{C_1, C_2, \ldots, C_n\}$, where $C_i = (P_i, A_i)$

    \For{$i = 1$ to $n$}
    \State Compute $z_i = (C_{i-1} \circ \ldots \circ C_1)(x)$ \Comment{Input to block $C_i$}
    \State Compute $y_i = A_i(P_i(z_i; \theta_i))$ \Comment{Output of block $C_i$}
    \State Compute $S_i = \varphi(y_i)$ \Comment{Scalar function for the block}

    \State Compute the gradient $g_i = \nabla_{\theta_i} S_i$

    \State Initialize $LH_i \in \mathbb{R}^{p_i \times p_i}$ as a zero matrix

    \For{$j = 1$ to $p_i$}
    \If{$g_{i,j}$ is not constant with respect to $\theta_i$}
    \State Compute $LH_i[j,:] = \nabla_{\theta_i} g_{i,j}$
    \Else
    \State $LH_i[j,:] = \vec{0}$
    \EndIf
    \EndFor
    \EndFor

    \State \Return $\{LH_1, LH_2, \ldots, LH_n\}$
  \end{algorithmic}
\end{algorithm}

The graphical representation of the algorithm is shown in Fig.~\ref{fig:layerwise_hessian}. This algorithm allows for
the efficient computation of local Hessians for each layer of a neural network, focusing on the specific
functional transformations performed by each layer. The sequential computation of the Hessian rows ensures that
memory usage remains manageable, even for large models with millions of parameters.

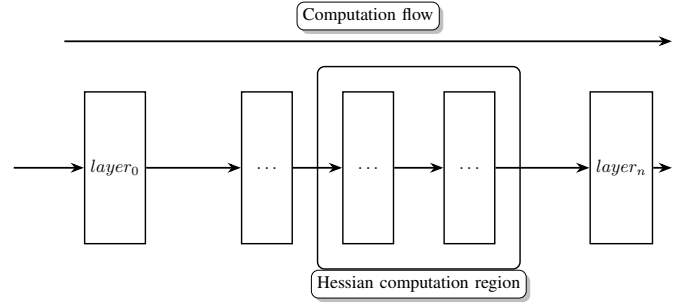
\begin{figure}[!t]
  \centering
  \resizebox{\columnwidth}{!}{%
    \begin{tikzpicture}[
        box/.style={draw, minimum width=1.2cm, minimum height=3cm, thick},
        smallbox/.style={draw, minimum width=1cm, minimum height=3cm, thick},
        arrow/.style={-Stealth, line width=1pt},
      ]
      % 1) Arrow above all layers
      \draw[arrow] (-2,4) -- (10,4);
      \node[draw, fill=white, rounded corners, drop shadow] at (4,4.5)
      {Computation flow};

      % 2) Hessian computation region
      \draw[rounded corners, thick] (3,-0.5) rectangle (7,3.5);

      % 3) Label under the block
      \node[draw, fill=white, rounded corners, drop shadow] at (5,-0.8) {Hessian computation region};

      % 4) Layers
      \node[box]      (L0) at (-1,1.5) {$layer_0$};
      \node[smallbox] (L1) at (2,1.5) {$\ldots$};
      \node[smallbox] (L2) at (4,1.5) {$\ldots$};
      \node[smallbox] (L3) at (6,1.5) {$\ldots$};
      \node[box]      (Ln) at (9,1.5) {$layer_n$};

      % 5) Arrows between layers
      \draw[arrow] (-3,1.5) -- (L0);
      \draw[arrow] (L0) -- (L1);
      \draw[arrow] (L1) -- (L2);
      \draw[arrow] (L2) -- (L3);
      \draw[arrow] (L3) -- (Ln);
      \draw[arrow] (Ln) -- (10,1.5);
    \end{tikzpicture}
  }
  \caption{Visualization of local Hessian computation for a neural network. The Hessian computation region
  indicates where layer-wise Hessians are calculated during forward propagation.}
  \label{fig:layerwise_hessian}
\end{figure}

\subsection{Mathematical Implementation Details}

\subsubsection{Gradient Computation}
In the context of automatic differentiation, the gradient $g_i$ is computed as:
\begin{equation}
  g_i = \nabla_{\theta_i} S_i = \frac{\partial S_i}{\partial y_i} \cdot \frac{\partial y_i}{\partial P_i}
  \cdot \frac{\partial P_i}{\partial \theta_i}
\end{equation}

where:
\begin{itemize}
  \item $\frac{\partial S_i}{\partial y_i} = \nabla_{y_i} \varphi(y_i)$ is the gradient of the aggregation function,
  \item $\frac{\partial y_i}{\partial P_i} = \nabla_{P_i} A_i(P_i)$ is the Jacobian of the activation function,
  \item $\frac{\partial P_i}{\partial \theta_i}$ is the Jacobian of the parameterized transformation with
    respect to its parameters.
\end{itemize}

\subsubsection{Computation of Hessian Matrix Rows}
For each component $j$ of the gradient $g_i$, its gradient with respect to $\theta_i$ is computed as:
\begin{equation}
  H_i[j,:] = \nabla_{\theta_i} g_{i,j} = \nabla_{\theta_i} \left( \frac{\partial S_i}{\partial \theta_{i,j}} \right)
\end{equation}

This requires repeated application of automatic differentiation to each component of the gradient.

\subsection{Spectral Analysis of the Local Hessian}

A detailed study of the spectrum of $LH_i$ provides an important tool for understanding the geometry of the
parameter space. The distribution and structure of the eigenvalues reflect key curvature properties of the
function, which is especially relevant for analyzing the conditioning of the optimization
problem~\cite{sagun2016, liaomahoney2021}.

For each $LH_i$, one can compute:

\begin{equation}
  LH_i = U_i\Lambda_i U_i^T = \sum_{j=1}^{p_i} \lambda_{i,j} u_{i,j} u_{i,j}^T
\end{equation}

where $\lambda_{i,j}$ is the $j$-th eigenvalue and $u_{i,j}$ is the corresponding eigenvector.

Characteristic spectral indicators of the Hessian include:

\begin{itemize}
  \item \textbf{Trace of the Hessian:} $\text{tr}(LH_i) = \sum_{j=1}^{p_i} \lambda_{i,j}$ - the sum of the eigenvalues
  \item \textbf{Determinant of the Hessian:} $\det(LH_i) = \prod_{j=1}^{p_i} \lambda_{i,j}$ - the product of
    the eigenvalues
\end{itemize}

Of particular interest is the observation of the distribution of eigenvalues across network layers and their
evolution during training, which enables tracking the changes in the geometry of the parameter space.

The distribution of eigenvalues of $LH_i$ provides information about the geometry of the functional space of
a layer. In particular, a concentration of eigenvalues near zero indicates the presence of manifolds of equal
function values (plateaus), which complicates optimization by gradient-based methods.

\begin{itemize}
  \item \textbf{Small networks:} If the network is too small (few layers $d_l$ or narrow layers), the linear
    transformation $z^{(l)}=W^{(l)}a^{(l-1)}+b^{(l)}$ becomes poorly expressed, leading to a large mean
    $\mathbb{E}[|z^{(l)}|]\gg1$ and shifting $z^{(l)}$ into the saturation regions of activation functions
    such as sigmoid or $\tanh$, where $f'(z)\approx0$. This exacerbates the vanishing gradient problem and
    significantly affects the structure of $LH_i$, concentrating its eigenvalues near zero.
  \item \textbf{Overly large networks:} An excessively large network can ``memorize'' noisy details of the
    data, resulting in an increase in the norm of the weights $||W^{(l)}||\gg1$ and shifting $z^{(l)}$ into
    the saturation region. Under such conditions, $LH_i$ also acquires a specific structure with many very
    small eigenvalues, reflecting excessive freedom in the parameter space.
  \item \textbf{Inappropriate architecture or inputs:} The architecture may be unsuitable for the specific
    properties of the data (high nonlinearity, variability of distributions, multidimensional dependencies).
    Poorly normalized or noisy inputs further increase the spread of $z^{(l)}$, intensifying saturation. As a
    result, neurons ``freeze'' and cease to participate effectively in training.
\end{itemize}

% *** RESEARCH METHODOLOGY ***
\section{Research Methodology}
\label{sec:methodology}

This section describes the comprehensive experimental framework employed to validate our theoretical findings
and investigate the practical utility of local Hessian analysis. We detail the experimental design, data
collection procedures, and analytical methods used to extract meaningful insights from the collected data.

\subsection{Experimental Design}

To assess the spectral properties of local Hessians, a comprehensive analysis was conducted on neural
networks of various architectures across a set of 37 datasets (22 for classification tasks and 15 for
regression). For each dataset, the following parameters were varied:

\begin{itemize}
  \item Number of layers and neurons in the networks
  \item Weight initialization methods
  \item Optimization algorithms (Adam, SGD, RMSProp)
  \item Activation functions (ReLU, Sigmoid, Tanh)
\end{itemize}

The total number of parameters in the studied models ranged from 13 to 9 million, and the number of layers
varied from 1 to 5, enabling the analysis of networks with different levels of parameterization. For both
tasks, the classical multilayer perceptron was used as the primary architecture. Cross-entropy loss was
employed for classification tasks, and mean squared error was used for regression tasks. Hyperparameters for
each dataset were selected empirically.

\subsection{Experimental Data Collection}

A specialized system was developed to monitor the evolution of internal characteristics of neural networks
during training. The experiments followed the methodology outlined below:

\begin{enumerate}
  \item For each dataset, three model variants were trained:
    \begin{itemize}
      \item A model with a small number of parameters (type ``no'')
      \item A model with a moderate number of parameters (type ``sure'')
      \item A model with a large number of parameters (type ``huge'')
    \end{itemize}

  \item During training, at each checkpoint iteration, the following data were recorded:
    \begin{itemize}
      \item Model weights and their spectral characteristics (distribution, statistics)
      \item Parameter gradients and their spectral characteristics
      \item $LH_i$ matrices for all layers and their eigenvalues
      \item Model quality metrics (for classification: Accuracy, Precision, Recall, F1, AUC; for regression:
        R2, MAE, RMSE)
      \item Value of the loss function on the training set
    \end{itemize}
\end{enumerate}

Special attention was paid to the computation of $LH_i$, for which a custom efficient algorithm with
component-wise computation and memory optimization was employed. This enabled the calculation of Hessians
even for models with a large number of parameters.

For each model, between 50 and 150 checkpoints were collected depending on the convergence rate, resulting in
a total of approximately 1500 network state snapshots with an aggregate size of about 50 gigabytes.

\subsection{Methodology and Experimental Data Processing}

A multi-stage approach was employed to analyze the collected experimental data, incorporating the following methods:

\begin{enumerate}
  \item \textbf{Correlation analysis:} Calculation of Pearson correlation coefficients between model
    parameters and quality metrics, with results visualized using heatmaps.

  \item \textbf{Spectral analysis:} Investigation of the eigenvalue distributions of weights and $LH_i$,
    including computation of statistical characteristics (mean, standard deviation, extrema).

  \item \textbf{Canonical Correlation Analysis (CCA):} Examination of nonlinear relationships between groups
    of quality metrics (Accuracy, Precision, F1, etc.) and internal network parameter characteristics. All
    data were standardized to ensure the validity of the analysis.

  \item \textbf{Visualization:} Use of heatmaps, scatter plots, and other graphical representations for clear
    interpretation of the results.

  \item \textbf{Statistical tests:} Assessment of the distributions of characteristics using the Shapiro-Wilk
    test~\cite{shapiro1965} to identify deviations from normality.
\end{enumerate}

Special attention was paid to the spectral analysis of $LH_i$ and their correlation with model quality
metrics. This comprehensive approach enabled the identification of not only direct correlations between
individual parameters, but also more complex relationships between groups of parameters.

% *** RESEARCH RESULTS ***
\section{Experimental Results}
\label{sec:results}

This section presents the comprehensive empirical findings from our large-scale experimental study across
111 experiments on 37 datasets. We analyze the spectral properties of local Hessians and their relationship
to network performance, architectural characteristics, and generalization ability.

\subsection{Primary Experimental Findings}

Experiments were conducted on computational clusters utilizing GPUs to accelerate calculations. The analysis
employed Python libraries such as NumPy, SciPy, Matplotlib, and Seaborn. Data collection required
approximately 40 hours, including model training and computation of $LH_i$. This duration was primarily due
to the initial, non-optimized implementation of the local Hessian, which necessitated recalculating all
gradients and Hessians at each training iteration. Subsequently, an optimized version of the algorithm was
developed, significantly reducing computation time.

Based on the collected data, several key observations can be highlighted:
\begin{itemize}
  \item Spectral analysis of $LH_i$ provides valuable insights into the internal structure and functioning of
    neural networks. Networks with different architectures exhibit characteristic patterns in the spectral
    properties of their Hessians, although the formal interpretation of all observed phenomena remains an open question.
  \item A more detailed analysis of the dynamics of $LH_i$ evolution across layers is required. It is
    advisable to enhance the use of canonical correlation analysis (CCA) and consider the application of
    factor analysis to uncover latent dependencies. Current analytical methods do not always reveal all
    possible relationships between parameters and quality metrics. Often, these issues indicate the presence
    of multiple architectural problems simultaneously, which may not be directly related. Such tools are
    limited in diagnosing network weaknesses, especially when training hyperparameters are suboptimal.
  \item Additionally, it is recommended to investigate the geometry of the gradient flow and Hessian
    properties in the context of differential geometry, as well as to consider modeling the gradient flow on
    the parameter manifold in the limiting case of activation function saturation.
\end{itemize}

\subsection{Comparative Analysis of Architectural Solutions Based on CCA}

The investigation of three types of architectures with varying parameterization (hereinafter referred to as
``no''-small, ``sure''-moderate, and ``huge''-overparameterized) revealed substantial differences in the
spectral properties of their $LH_i$. To identify relationships between model parameters and their
performance, canonical correlation analysis (CCA) was employed, enabling the establishment of correlations
between two groups of variables:

\begin{itemize}
  \item \textbf{Group A}: model quality metrics-Accuracy, Precision, Recall, F1, AUC, and train\_loss
  \item \textbf{Group B}: neural network parameters-weights, gradients, Hessian eigenvalues, and their
    spectral characteristics
\end{itemize}

The CCA analysis revealed the dependencies summarized in Table~\ref{tab:cca_max}.

\begin{table}[!t]
  \renewcommand{\arraystretch}{1.3}
  \caption{CCA Correlation Statistics Across Architectures}
  \label{tab:cca_max}
  \centering
  \begin{tabular}{|l|c|c|c|}
    \hline
    \textbf{Statistic} & \textbf{no} & \textbf{sure} & \textbf{huge} \\
    \hline
    max & 0.406 & 0.349 & 0.429 \\
    avg & -0.955 & 0.099 & 0.220 \\
    median & 0.182 & 0.182 & 0.189 \\
    min & -0.965 & -0.770 & 0.149 \\
    std & 0.976 & 0.277 & 0.082 \\
    \hline
  \end{tabular}
\end{table}

These data are clearly illustrated by the CCA Score statistics comparison plot
(Fig.~\ref{fig:cca_score_comparison}), which demonstrates significant differences in minimum values between
architectures and, in particular, highlights the high variability of the moderate architecture (``sure'')
with a minimum value around -0.77. While this is higher than that of the small architecture (``no''), it
remains considerably lower than that of the large architecture (``huge'').

\begin{figure}[!t]
  \centering
  \includegraphics[width=\columnwidth]{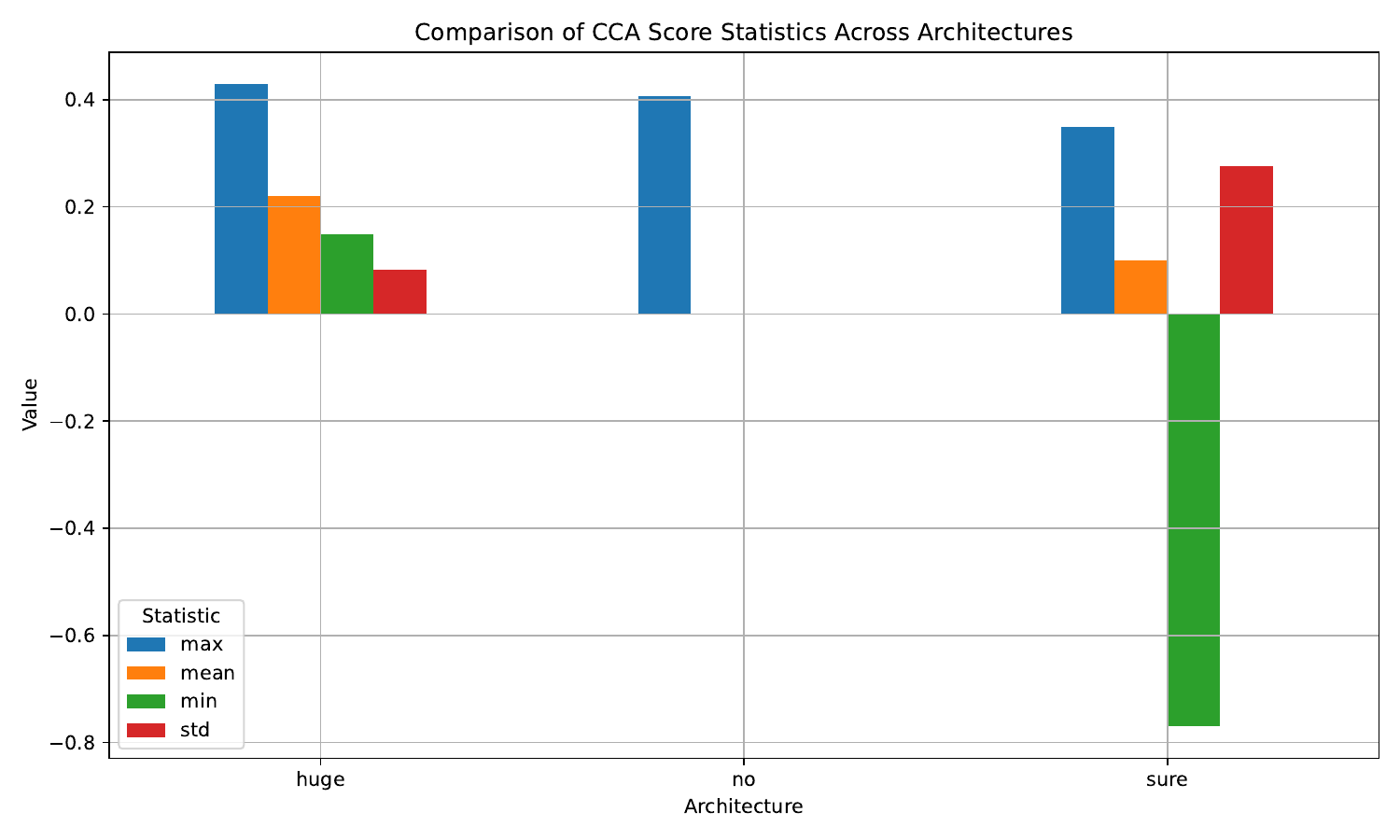}
  \caption{Comparison of CCA Score statistics across architectures. Large architectures (``huge'') exhibit
  the highest stability with standard deviation of 0.082, while small architectures (``no'') show extreme
  variability (std=0.976).}
  \label{fig:cca_score_comparison}
\end{figure}

The results indicate that large architectures (``huge'') exhibit the highest and most stable CCA correlation
values (standard deviation of only 0.082), suggesting a more robust relationship between internal network
parameters and prediction quality. In contrast, small architectures (``no'') display extreme negative
outliers in both mean (-0.955) and minimum (-0.965) values, as well as high variability (standard deviation
0.976), reflecting the instability of their functional behavior.

\subsection{Spectral Characteristics of Gradients in Different Architectures}

To analyze the spectral characteristics of gradients, the Welch method was employed to estimate the power
spectral density (PSD) of gradients across layers. The following parameters were used for the Welch method:
\begin{itemize}
  \item Window length: 256
  \item Overlap step: 128
  \item Window function: Hanning
\end{itemize}

Spectral analysis of the gradients in the third layer (Fig.~\ref{fig:layer3_gradient_spectral}) revealed
dramatic differences between the studied architectures. The maximum PSD values obtained via the Welch method
for the ``huge'' architecture exceed those of the ``no'' architecture by more than two orders of magnitude,
reaching values on the order of $1.2\times10^6$. The mean values also show a significant increase from small
to large architectures, indicating a fundamentally different structure of the gradient space in highly
parameterized models.

\begin{figure}[!t]
  \centering
  \includegraphics[width=\columnwidth]{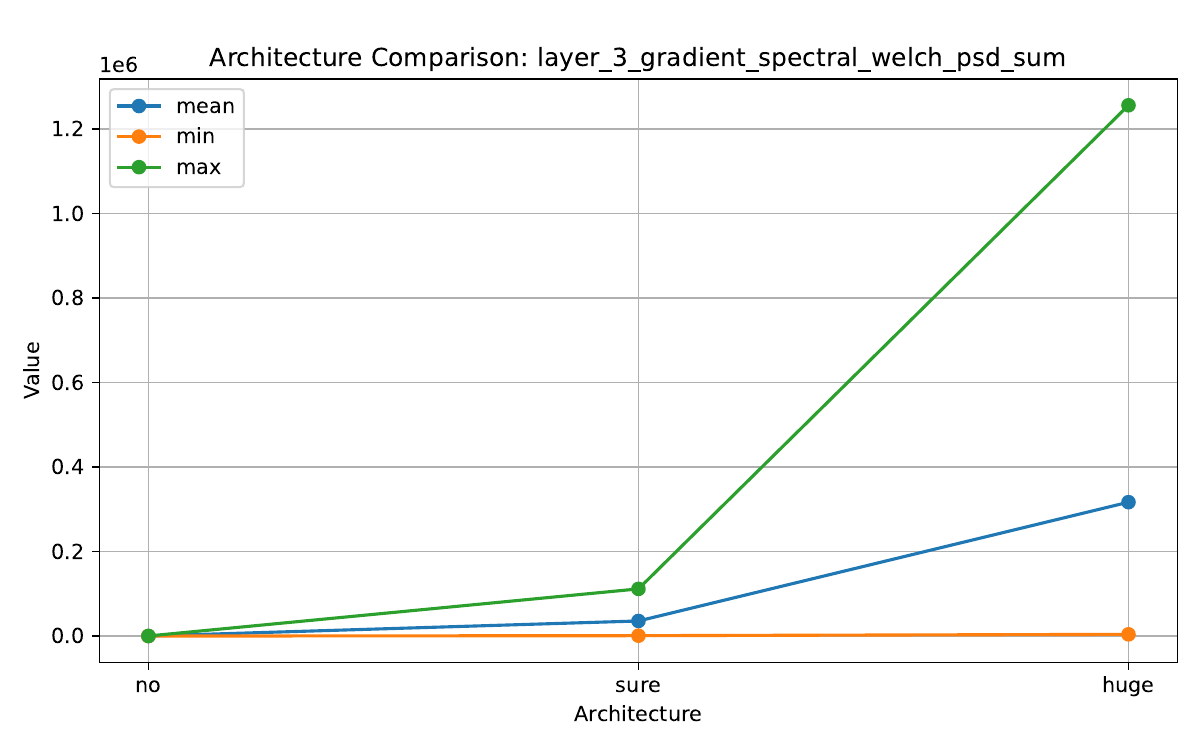}
  \caption{Comparison of spectral characteristics of third-layer gradients. The ``huge'' architecture shows
  PSD values exceeding the ``no'' architecture by over 100 times, indicating qualitatively different gradient
  propagation dynamics.}
  \label{fig:layer3_gradient_spectral}
\end{figure}

Such a scale of differences indicates a qualitative change in the nature of gradient propagation in large
architectures, where high-frequency components with substantially greater energy are formed. These
observations are consistent with the notion that overparameterization facilitates the emergence of a more
complex functional surface structure with numerous local features.

\subsection{Distribution of Canonical Weights Across Architectures}

The analysis of the distribution of canonical X and Y weights revealed structural features in how network
parameters influence performance. The visualizations presented in Figs.~\ref{fig:cca_x_weights}
and~\ref{fig:cca_y_weights} demonstrate substantial differences in the weight distributions between architectures.

\begin{figure}[!t]
  \centering
  \includegraphics[width=\columnwidth]{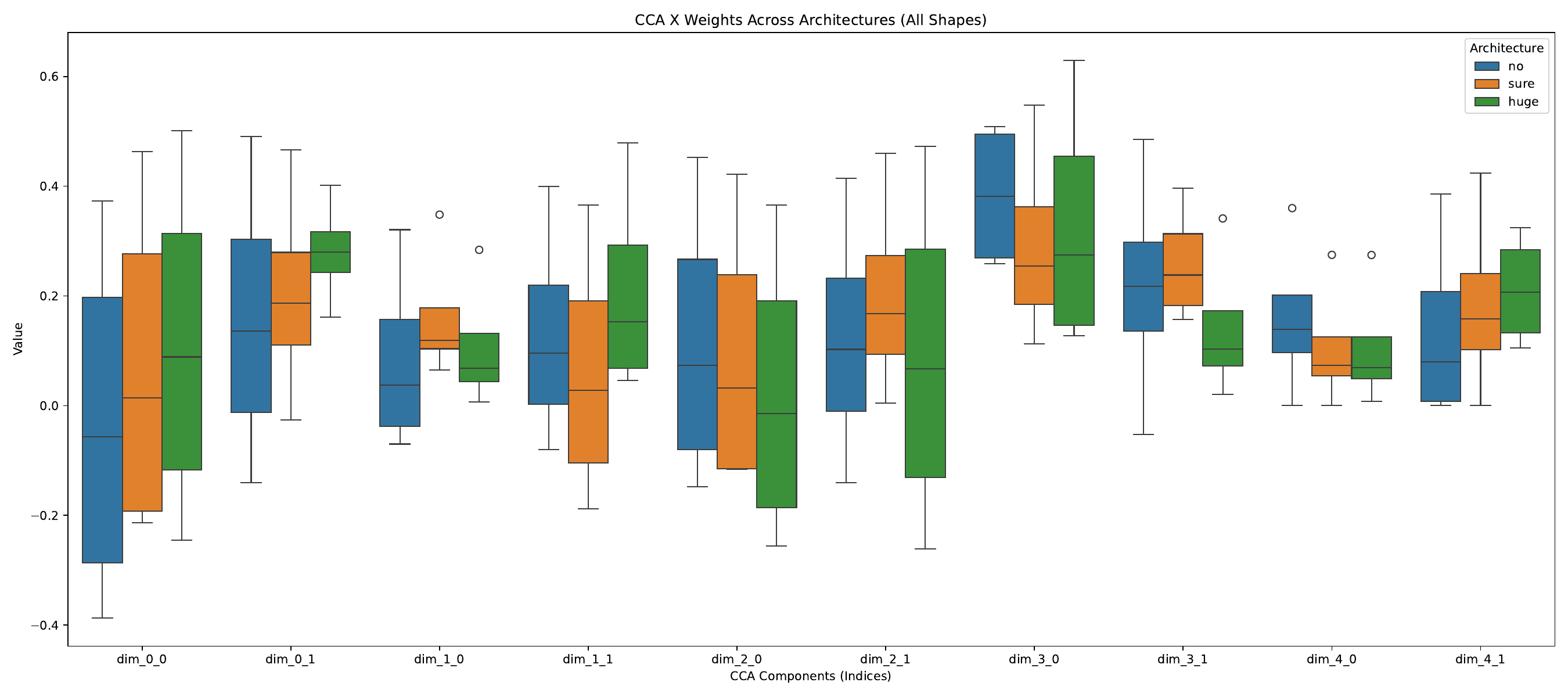}
  \caption{Distribution of canonical X-weights across architectures. Large architectures (``huge'') show
  more uniform distribution, while small architectures (``no'') exhibit concentrated structure with
  dominant components.}
  \label{fig:cca_x_weights}
\end{figure}

\begin{figure}[!t]
  \centering
  \includegraphics[width=\columnwidth]{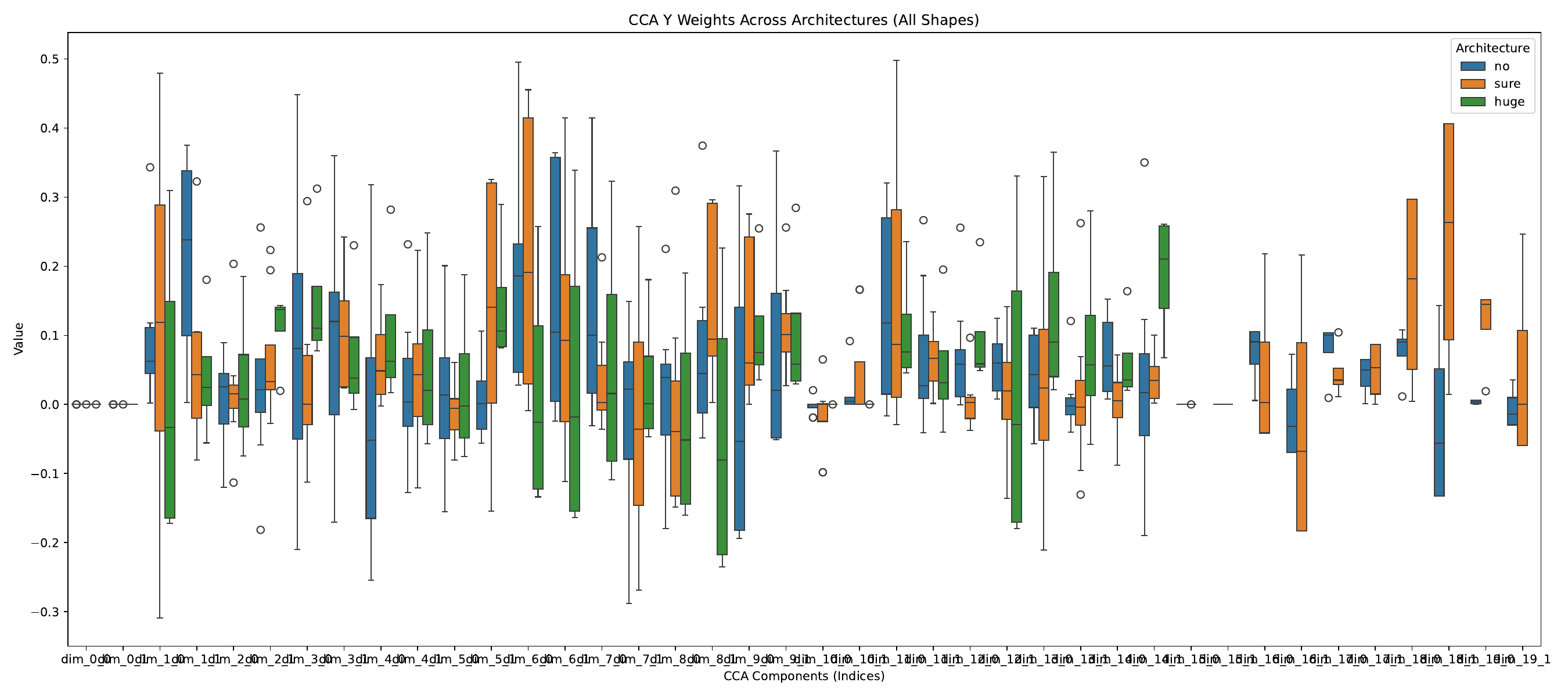}
  \caption{Distribution of canonical Y-weights across architectures. Large architectures display denser,
  centered distribution around zero, while small architectures show significant asymmetry with extreme outliers.}
  \label{fig:cca_y_weights}
\end{figure}

The following pattern is observed: for X-weights (linking quality metrics to canonical variables), large
architectures (``huge'') exhibit a more uniform distribution across the entire spectrum of components,
indicating a more balanced utilization of the full set of parameters to achieve high performance. In
contrast, small architectures (``no'') are characterized by a more concentrated structure with several
dominant components, suggesting that certain parameters are ``overstressed'' to achieve the desired outcome.

For Y-weights, which connect neural network parameters to canonical variables, an even more pronounced
differentiation is observed: in large architectures, the distribution is denser and centered around zero with
minor outliers, whereas in small architectures, there is significant asymmetry with extreme values at the
distribution tails. This supports the hypothesis that in underparameterized models, individual parameters
bear a disproportionately high load, reducing the model's robustness to input variations.

\subsection{Dynamics of Canonical Weight Coefficients}

The analysis of canonical weight coefficients revealed significant differences between architectures. Each
canonical vector represents a linear combination of the original variables that maximizes the correlation
between groups A and B. Particularly notable are the differences in the structure of the Y-component weights,
which link network parameters to quality metrics (see Fig.~\ref{fig:feature_importance}).

\begin{figure}[!t]
  \centering
  \includegraphics[width=\columnwidth]{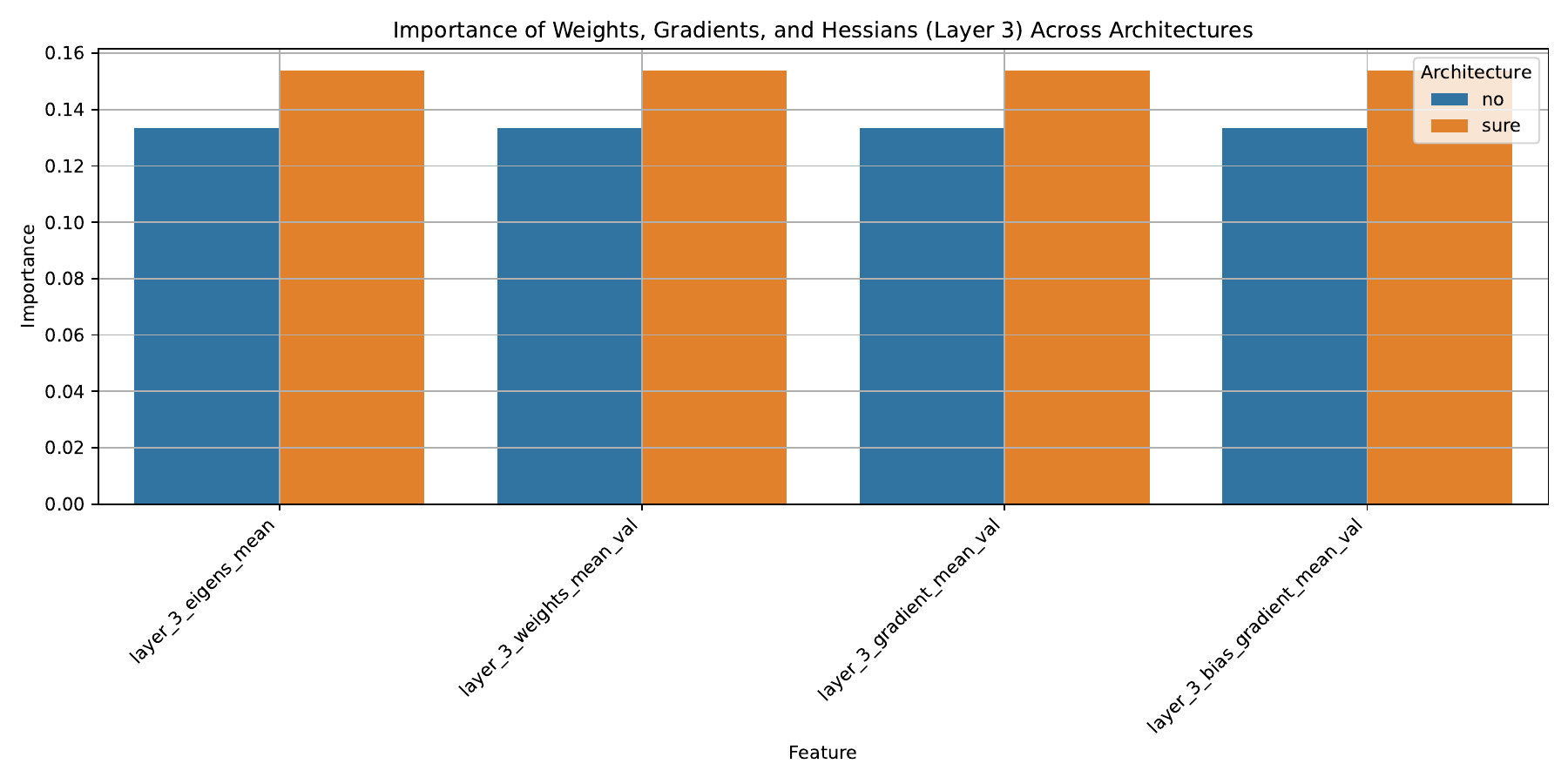}
  \caption{Importance of group B parameters (weights, gradients, and Hessians) of the third layer for
  different architectures. Small architectures show gradient dominance, while large architectures exhibit
  greater importance of Hessian eigenvalues.}
  \label{fig:feature_importance}
\end{figure}

In small architectures, there is a pronounced dominance of weights associated with the gradient component,
whereas in large architectures, the eigenvalues of the Hessian become more important. This observation
supports the theoretical hypothesis that, in underparameterized models, the learning dynamics are primarily
determined by local gradients, while in overparameterized models, the curvature of the functional surface
plays a more significant role.

A detailed analysis of the differentiating weight parameters of group B (Table~\ref{tab:weight_diff})
revealed extreme differences between architectures:

\begin{table}[!t]
  \renewcommand{\arraystretch}{1.3}
  \caption{Group B Parameters with Largest Differences Between Architectures}
  \label{tab:weight_diff}
  \centering
  \begin{tabular}{|c|c|c|c|c|}
    \hline
    \textbf{Huge} & \textbf{no} & \textbf{sure} & \textbf{huge - no} & \textbf{huge - sure} \\
    \hline
    4384.0 & 3.42 & 212.92 & 4380.58 & 4171.08 \\
    1208.0 & 30.85 & 376.31 & 1177.15 & 831.69 \\
    620.67 & 3.88 & 37.54 & 616.78 & 583.13 \\
    196.0 & 12.0 & 46.0 & 184.0 & 150.0 \\
    \hline
  \end{tabular}
\end{table}

These data demonstrate radical differences in parameter magnitudes between large and small architectures,
reaching up to three orders of magnitude (4380.58). Such a contrast indicates a qualitatively different
operating regime in overparameterized networks, where the accumulation of weights can reach significant
values without negatively affecting prediction quality, due to compensatory effects between layers.

\subsection{Statistical Properties of Canonical Correlation Weights}

An additional analysis of the statistical properties of CCA weights (Table~\ref{tab:cca_stats}) revealed an
interesting asymmetry in the distribution of variability across dimensions:

\begin{table}[!t]
  \renewcommand{\arraystretch}{1.3}
  \caption{Statistical Properties of CCA Weights for Different Architectures and Shapes}
  \label{tab:cca_stats}
  \centering
  \begin{tabular}{|l|c|c|c|c|}
    \hline
    \textbf{Arch.} & \textbf{Type} & \textbf{Shape} & \textbf{Var.(d0)} & \textbf{Var.(d1)} \\
    \hline
    no & X & (5,2) & 0.1899 & 0.1315 \\
    sure & X & (5,2) & 0.2159 & 0.1274 \\
    huge & X & (5,2) & 0.2064 & 0.1553 \\
    \hline
    no & Y & (15,2) & 0.0000 & 0.1292 \\
    sure & Y & (15,2) & 0.0000 & 0.1671 \\
    huge & Y & (15,2) & 0.0000 & 0.0643 \\
    \hline
    no & Y & (20,2) & 0.0000 & 0.0665 \\
    sure & Y & (20,2) & 0.0000 & 0.0003 \\
    \hline
  \end{tabular}
\end{table}

It is noteworthy that the variance for the first dimension of Y-weights is practically zero across all
architectures, indicating a strict structural relationship between network parameters and quality metrics
along certain directions. In contrast, the X-weights, which reflect the contribution of neural network
parameters to the canonical variables, exhibit substantial variance in both dimensions, suggesting greater
flexibility in the formation of the network's internal representations.

It should also be noted that the medium-sized architecture (``sure'') demonstrates the highest variance of
Y-weights for the (15, 2) shape, but a sharply reduced variance for the (20, 2) shape, which may indicate
more efficient parameter utilization compared to other architectures.

\subsection{Clustering of Architectures by Spectral Properties}

Of particular interest are the results of architectural analysis after dimensionality reduction using PCA
(see Fig.~\ref{fig:architecture_clustering}).

\begin{figure}[!t]
  \centering
  \includegraphics[width=\columnwidth]{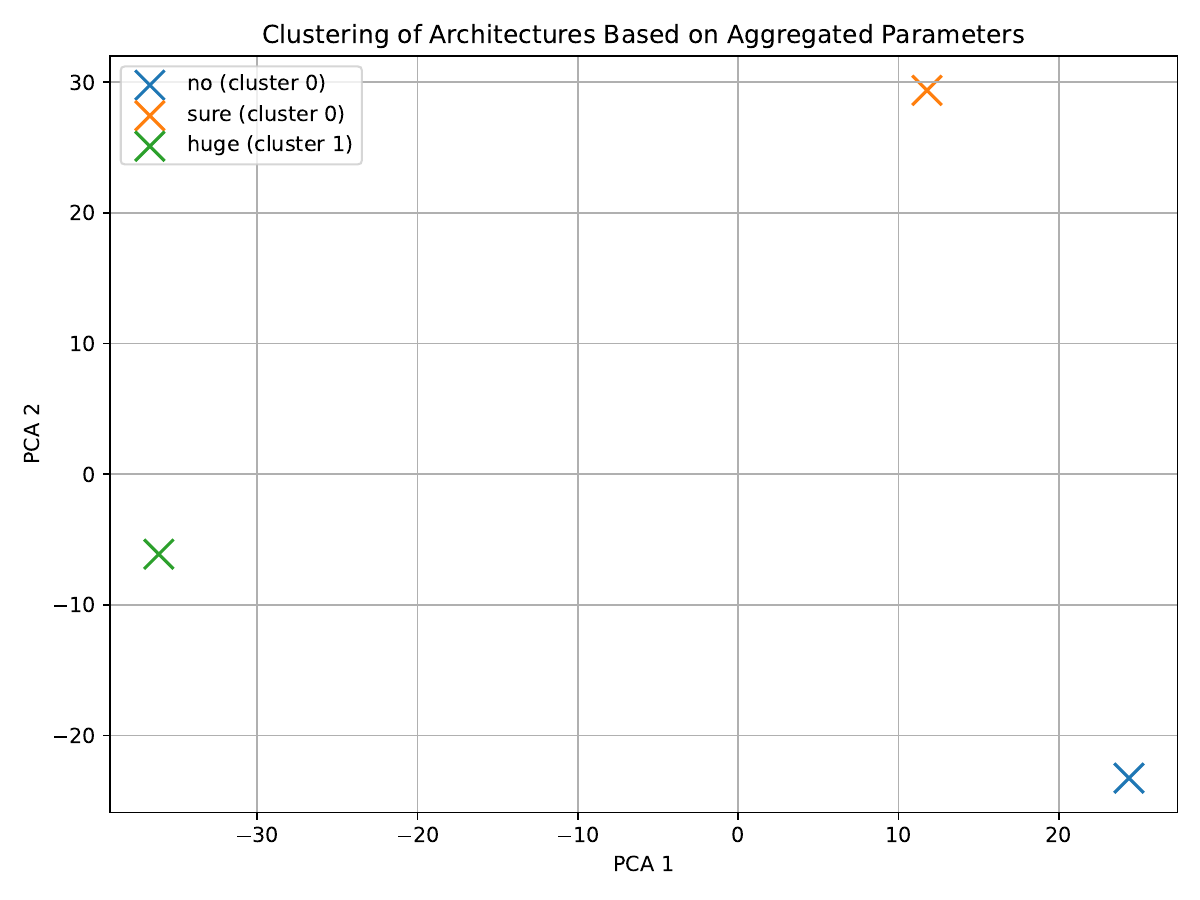}
  \caption{Distribution of architectures in the space of spectral characteristics of Hessians after
  dimensionality reduction. Three distinct clusters correspond to small (``no''), medium (``sure''), and
  large (``huge'') architectures, indicating qualitative differences in their parameter spaces.}
  \label{fig:architecture_clustering}
\end{figure}

The obtained results demonstrate a clear separation of all three architectures in the principal component
space: the small (``no''), medium (``sure''), and large (``huge'') architectures form three distinct groups
of points, significantly distant from each other. This indicates substantial differences in their parameter
spaces and functional behavior.

Such pronounced separation points to the existence of several ``complexity thresholds,'' the crossing of
which leads to qualitative changes in the functional behavior of the network. The most dramatic transition is
observed between the ``sure'' and ``huge'' architectures, supporting the hypothesis that the ``huge''
architecture operates in a fundamentally different regime, characterized by a specific distribution of the
spectral properties of Hessians and their relationship with quality metrics.

\subsection{Relationship to Knowledge Transfer and Generalization}

One of the most significant findings of this study is the established connection between the spectral
structure of local Hessians and the model's generalization capability. Models exhibiting a more uniform
distribution of Hessian eigenvalues-without pronounced peaks and with less concentration near
zero-demonstrate superior performance on test datasets.

Analysis of the CCA weight statistics (Table~\ref{tab:cca_stats}) indicates that large architectures
(``huge'') display a more balanced distribution of X-weights along the first dimension (0.2064) compared to
small architectures (0.1899), which correlates with their enhanced generalization ability. At the same time,
the standard deviation of the defining parameter values in large architectures is significantly higher,
suggesting a greater capacity for fine-grained feature differentiation.

This pattern is observed regardless of the absolute number of model parameters, supporting the central
hypothesis of this research: local properties of layer characteristics are more critical for generalization
capability than the total number of parameters.

\subsection{Quantitative Assessment of Differences Between Architectures}

A particularly important aspect is the quantitative assessment of differences between architectural
solutions. The analysis of the top differentiating parameters of group B (see Table~\ref{tab:weight_diff})
shows that the differences in weights between large and small architectures can reach up to three orders of
magnitude (4380.58 for the parameter with the largest difference).

The plot of the spectral characteristics of the third layer gradients
(Fig.~\ref{fig:layer3_gradient_spectral}) clearly demonstrates that the differences in maximum values between
the ``huge'' and ``no'' architectures can exceed a factor of one hundred, reaching absolute values on the
order of $1.2\times10^6$. Such a scale of differences indicates a fundamentally different nature of gradient
propagation in large architectures, where high-energy spectral components are formed.

It is noteworthy that the largest differences are observed between the ``huge'' and ``no'' architectures
(4380.58), while the differences between ``huge'' and ``sure'' (4171.08) are only slightly smaller. This
points to the existence of a ``complexity barrier,'' the crossing of which leads to a qualitative change in
the network's operating regime.

Such extreme differences in weights do not necessarily lead to model performance degradation, which
contradicts intuitive expectations. On the contrary, large models with extreme weight values demonstrate
higher result stability, as evidenced by the CCA correlation statistics (standard deviation of 0.082 for
``huge'' versus 0.976 for ``no'').

\subsection{Observations and Effects}

Several unexpected effects were discovered during the analysis:

\begin{enumerate}
  \item \textbf{Shift of CCA weights across datasets.} The analysis of CCA weights revealed a significant
    shift in the structure of weights between different datasets, even for similar architectures. This
    confirms a strong dependence of the network's functional behavior on the data structure and highlights
    the necessity of adapting the architecture to the specific task.

  \item \textbf{Nonlinear dependence of stability on architecture size.} Contrary to expectations, large
    architectures (``huge'') exhibit more stable spectral characteristics (lower standard deviation) than
    medium-sized ones (``sure''), which contradicts the intuitive assumption that overparameterization should
    lead to greater variability.

  \item \textbf{Asymmetry in the variance distribution of CCA weights.} Y-weights have almost zero variance
    along the first dimension for all architectures, but significant variance along the second dimension.
    This indicates the existence of structural constraints in the way network parameters affect quality metrics.

  \item \textbf{Opposite behavior of X and Y weights in distribution.} X-weights demonstrate a more compact
    structure with less variation between architectures, whereas Y-weights show significant differences both
    in the shape of the distribution and in the range of values, especially for extreme components.
\end{enumerate}

In addition to the listed results, the analysis revealed a number of correlation patterns directly related to
the rank of weight matrices, spectral characteristics of the Hessian, and the network's generalization ability:

\begin{enumerate}
  \item \textbf{Correlation between the rank of weights and Hessian with generalization quality.} There is a
    notable relationship between the reduced rank of weight matrices and the corresponding local Hessian and
    signs of overfitting and redundancy. These findings indicate a deterioration in the network's ability to
    generalize to unseen data when the rank is low.
    
  \item \textbf{Layer Hessian as an indicator of overfitting.} A sparse local Hessian or one with a
    predominant cluster of small eigenvalues correlates with insufficient generalization capacity of the
    layer. This effect is particularly pronounced in the final layers.
    
  \item \textbf{Symmetry of the Hessian spectrum and saddle points.} An almost symmetric distribution of
    Hessian eigenvalues around zero is associated with saddle points. Such points are often accompanied by
    small gradient norms.
    
  \item \textbf{Similarity of weights in adjacent layers.} In well-tuned networks, the weight matrices of
    neighboring layers exhibit pronounced similarity (in terms of SVD or spectrum), which can be interpreted
    as coordinated feature processing.
\end{enumerate}

\subsection{Practical Implications for Architecture Optimization}

The conducted analysis allows us to formulate several practical recommendations for optimizing neural network
architectures:

\begin{enumerate}
  \item \textbf{Optimal parameter allocation across layers.} The results show that a significant increase in
    the number of parameters in deep layers relative to the initial ones leads to the formation of high peaks
    in the Hessian spectrum, which may indicate overfitting in these layers. A more balanced parameter
    distribution is recommended.
    
  \item \textbf{Detection of insufficient expressivity.} Low values of the largest Hessian eigenvalues in the
    initial layers (observed in small architectures, ``no'') may serve as an indicator of insufficient model
    expressivity. In such cases, it is advisable to increase the number of parameters specifically in these layers.
    
  \item \textbf{Overfitting detection.} A high concentration of eigenvalues near zero in the final layers
    indicates overfitting and may signal the need for additional regularization or a reduction in the number
    of parameters in these layers.
    
  \item \textbf{Optimizer adaptation.} The structure of the Hessian spectrum can be used to adapt optimizer
    hyperparameters. For example, a high ratio of the largest to the smallest eigenvalue (the condition
    number) indicates the necessity of using adaptive optimization methods.
\end{enumerate}

\subsection{Investigation of the Internal Structure of Neural Network Layers}

Spectral analysis of local Hessians provides insights into the functional roles of layers:
\begin{itemize}
  \item Layers with a distributed spectrum without pronounced concentration perform complex nonlinear transformations.
  \item Layers with a peak near zero likely indicate overparameterization or activation saturation.
  \item A correlation is observed between the presence of several dominant eigenvalues and the layer's
    ability to extract key features.
\end{itemize}

These findings help to elucidate how the network solves the task and what transformations occur at different
levels of the hierarchy.

\subsection{Architecture Classification via Snapshots}

The obtained snapshots of weights, local Hessians, and gradients during training reveal several
characteristic patterns that could potentially be used for classifying neural network architectures. However,
these patterns do not provide a definitive conclusion regarding how exactly they influence generalization
quality or in which specific layer(s) of the network these effects manifest. Several attempts were made to
train a classifier based on these patterns, but the results were not satisfactory.

Nevertheless, by applying dimensionality reduction to the network snapshot data using the UMAP algorithm, an
interesting distribution of snapshots in the reduced space can be observed. Fig.~\ref{fig:umap} shows how
snapshots of weights, gradients, and Hessians are distributed in a two-dimensional space. Each point on the
plot corresponds to a single network snapshot, and the color indicates the network architecture.

\begin{figure}[!t]
  \centering
  \includegraphics[width=\columnwidth]{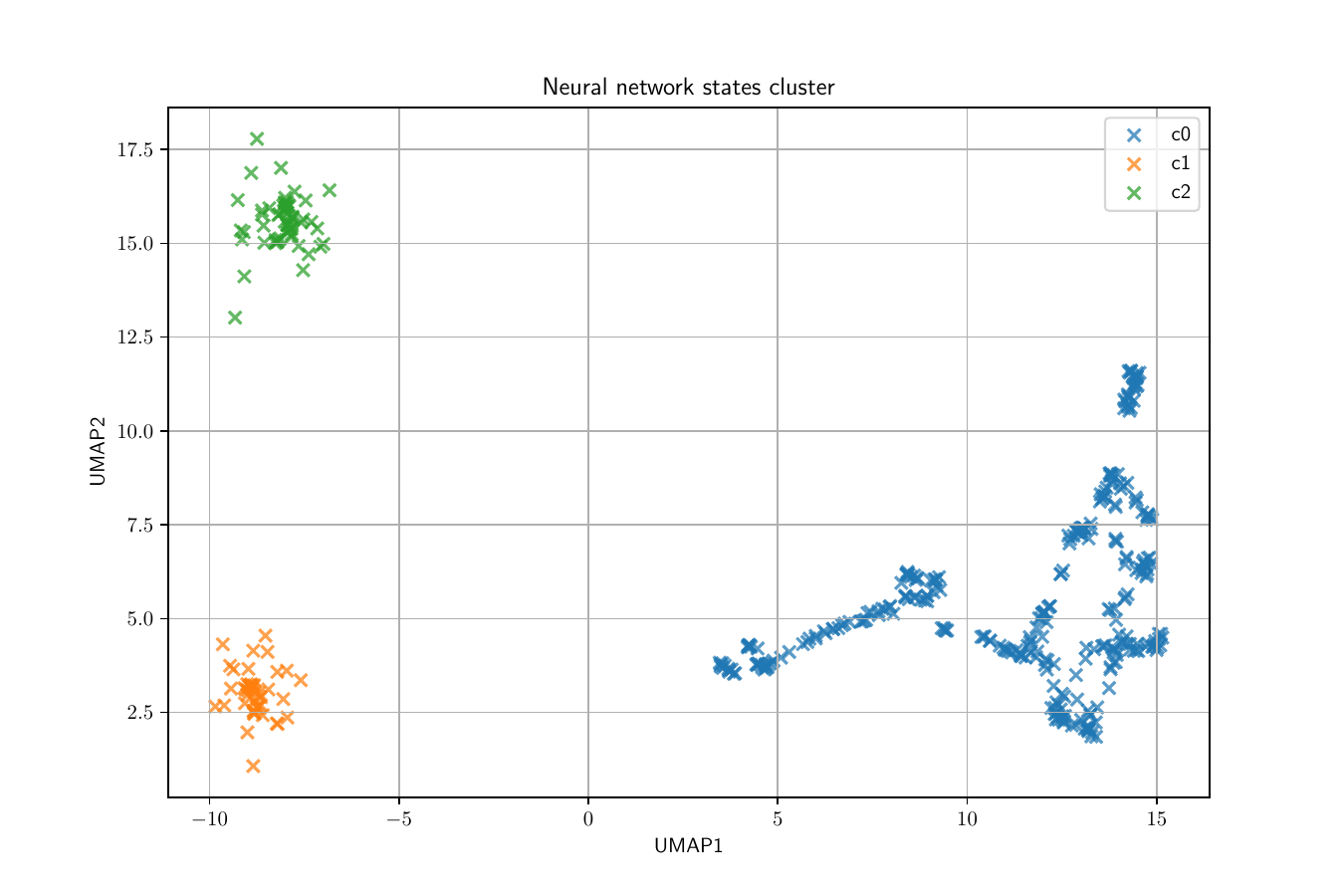}
  \caption{Distribution of snapshots of weights, gradients, and Hessians in a two-dimensional space after
  UMAP dimensionality reduction. Different architectures form distinct clusters, suggesting characteristic
  patterns in their internal representations.}
  \label{fig:umap}
\end{figure}

The plot demonstrates that snapshots from different architectures are distributed non-uniformly, forming
distinct clusters. This suggests that different architectures possess characteristic patterns in their
weights, gradients, and Hessians, which could be leveraged for classification. However, further research is
required for a more precise interpretation of these results.

\subsection{Generalized Practical Guidelines}

Based on the conducted research, the following practical recommendations can be formulated:
\begin{itemize}
  \item A reduced rank and symmetric spectrum of the Hessian should be regarded as warning signals, as they
    are associated with saddle point issues, overparameterization, and deteriorated generalization.
  \item Moderate spectral power and a sufficiently broad spectrum of eigenvalues are linked to the presence
    of a diversity of directions for network training and, consequently, a high potential for knowledge transfer.
  \item Bias parameters require regular monitoring: their contribution to the overall Hessian spectrum can
    serve as a metric for deviation from optimal geometry.
  \item Cross-correlation analysis of the spectra of weights, gradients, and the Hessian enables the
    identification of problematic directions and timely adaptation of the training plan (learning rate
    scheduling, application of second-order optimizers, addition of differential regularization, etc.).
\end{itemize}

Thus, in-depth spectral and rank analysis of local Hessians is a powerful diagnostic tool that allows for the
detection of hidden issues, formulation of recommendations for architecture and optimizer adjustment, and
quantitative assessment of the network's generalization capability.

% *** DISCUSSION ***
\section{Discussion}
\label{sec:discussion}

The conducted research provides a mathematically grounded tool for studying the internal dynamics of neural
networks, moving beyond empirical trial-and-error approaches. A key result is the established connection
between the geometric properties of the parameter space and the functional behavior of the network, as
highlighted in the main contributions of this work.

\begin{proposition}
  Analyzing a neural network as a composition of nonlinear operators or as a chaotic dynamical system yields
  informative insights into its internal structure, data processing mechanisms, and the mathematical
  constraints of its architecture.
\end{proposition}

Viewing the network as a dynamical system evolving on a high-dimensional manifold with nontrivial geometry
opens new avenues for understanding and improving learning methods.

Of particular interest is the development of the $LH_i$ concept in combination with Riemannian geometry,
aiming for a more detailed investigation of the local geometry of the parameter space. This approach may
facilitate the prompt identification of regions where the network's optimization process encounters difficulties.

The experimental results demonstrate that spectral properties of local Hessians serve as reliable indicators
of network health and performance. The clear separation of architectures in the spectral space
(Fig.~\ref{fig:architecture_clustering}) suggests that these properties capture fundamental differences in
how networks process information. The correlation between Hessian spectral characteristics and generalization
ability provides a theoretical foundation for architecture design principles.

Our findings also reveal interesting phase transitions in network behavior as parameterization increases.
The dramatic differences in gradient spectral density (over 100-fold, Fig.~\ref{fig:layer3_gradient_spectral})
between small and large architectures indicate qualitatively different optimization landscapes. This supports
recent theoretical work suggesting that overparameterized networks operate in fundamentally different regimes
than their smaller counterparts.

The practical implications of our work extend beyond diagnostic applications. The proposed metrics based on
local Hessian spectra can guide architecture search, inform regularization strategies, and provide early
warning signs of training problems. The ability to detect overfitting, underparameterization, and suboptimal
parameter distribution during training enables more efficient model development.

% *** CONCLUSION ***
\section{Conclusion}
\label{sec:conclusion}

This work introduces a novel approach for analyzing neural networks through the local properties of their
parameter space, investigated via layer-wise Hessian matrices. The proposed concept of the local Hessian enables:
\begin{itemize}
  \item Analysis of the geometry of the functional space of individual layers;
  \item Identification of patterns in the distribution of eigenvalues during training;
  \item Demonstration of the relationship between the Hessian spectrum, activation saturation, the formation
    of directions, and the evolution of representations.
\end{itemize}

It is important to emphasize the scale of the conducted experiment: approximately 1500 snapshots of various
network states were collected, totaling about 50 GB of data, which allowed for the identification of robust patterns.

The comprehensive empirical study across 111 experiments on 37 datasets reveals consistent relationships
between spectral properties of local Hessians and network performance. We have demonstrated that:

\begin{enumerate}
  \item Large architectures exhibit fundamentally different spectral characteristics compared to smaller ones,
    with more stable and balanced distributions of Hessian eigenvalues.
  \item The spectral structure of local Hessians correlates strongly with generalization capability,
    providing quantitative metrics for architecture quality assessment.
  \item Distinct patterns in gradient spectral density and canonical correlation weights can differentiate
    between architectures and predict their behavior.
  \item Local Hessian analysis enables early detection of architectural problems such as overfitting,
    underparameterization, and suboptimal parameter distribution.
\end{enumerate}

Future research directions may include:
\begin{itemize}
  \item Detailed investigation of the relationship between the spectral properties of Hessians and the
    functional characteristics of layers in the context of Riemannian geometry;
  \item Analysis of the dynamics of the spectrum during training and its connection to generalization ability
    across diverse tasks and domains;
  \item Application of the approach to novel architectures such as transformers, graph neural networks, and
    diffusion models;
  \item Development of automated methods for architecture optimization based on local Hessian spectral
    properties;
  \item Investigation of the connection between local Hessian structure and certified robustness properties;
  \item Extension of the framework to analyze pruning, quantization, and other model compression techniques.
\end{itemize}

The methodological framework presented in this work establishes a foundation for principled diagnostics and
informed design of neural network architectures based on their local geometric properties. By bridging
theoretical analysis with large-scale empirical validation, we provide both conceptual insights and practical
tools for the deep learning community.

% *** APPENDIX ***
\appendices

\section{Structure of Experimental Data}
\label{app:data}

During the research, for each neural network architecture at every training checkpoint, a snapshot of the model state was saved. Below is a detailed structure of such a snapshot:

\begin{itemize}
  \item \textbf{layer.X} -- information about layer X of the neural network:
    \begin{itemize}
      \item \textbf{weights} -- layer weight coefficients
      \item \textbf{weights\_spectral} -- spectral characteristics of the weights (mean, standard deviation, minimum, maximum, histogram, results of spectral analysis using the Welch method)
      \item \textbf{gradient} -- layer weight gradients
      \item \textbf{gradient\_spectral} -- spectral characteristics of the gradients
      \item \textbf{bias} -- layer bias parameters
      \item \textbf{bias\_spectral} -- spectral characteristics of the bias parameters
      \item \textbf{bias\_gradient} -- gradients of the bias parameters
      \item \textbf{bias\_gradient\_spectral} -- spectral characteristics of the bias gradients
      \item \textbf{hessian} -- local Hessian matrix of the layer
      \item \textbf{hessian\_spectral} -- spectral characteristics of the local Hessian
      \item \textbf{hessian\_eigens} -- eigenvalues of the local Hessian
      \item \textbf{hessian\_eigens\_spectral} -- statistical and spectral characteristics of the eigenvalues:
        \begin{itemize}
          \item \textbf{mean} -- mean value
          \item \textbf{std} -- standard deviation
          \item \textbf{min} -- minimum value
          \item \textbf{max} -- maximum value
          \item \textbf{histogram} -- distribution histogram
          \item \textbf{welch} -- results of spectral analysis using the Welch method
          \item \textbf{top\_peaks} -- main peaks in the spectrum
        \end{itemize}
      \item \textbf{hessian\_rank} -- rank of the Hessian matrix
      \item \textbf{hessian\_condition} -- condition number (ratio of the largest to the smallest eigenvalue)
    \end{itemize}

  \item \textbf{iteration} -- training iteration number

  \item \textbf{scores} -- model quality metrics:
    \begin{itemize}
      \item \textbf{Accuracy} -- classification accuracy
      \item \textbf{Precision} -- precision (proportion of correct positive predictions)
      \item \textbf{Recall} -- recall (proportion of detected positive cases)
      \item \textbf{F1} -- F1-score (harmonic mean of precision and recall)
      \item \textbf{AUC} -- area under the ROC curve
      \item \textbf{train\_loss} -- value of the loss function on the training set
    \end{itemize}
\end{itemize}

\section{Dataset Collection}
\label{app:datasets}

The following datasets were used in the experiments. A total of 37 datasets were employed: 22 for classification tasks and 15 for regression tasks.

\textbf{Classification datasets:} MNIST, CIFAR-10, Fashion-MNIST, CIFAR-100, KMNIST, EMNIST, Iris, Wine, Breast Cancer Wisconsin, Digits, SpamBase, Make Classification, Make Blobs, Titanic Dataset, Adult Income, Credit Card Fraud Detection, Make Biclusters, Make Checkerboard, Make Circles, Make Hastie 10 2, Make Moons, Make Multilabel Classification.

\textbf{Regression datasets:} Diabetes, Energy Efficiency, Airfoil Self-Noise, Concrete Compressive Strength, Make Regression, House Prices Dataset, Make Friedman1, Make Friedman2, Make Friedman3, Make Low Rank Matrix, Make S Curve, Make Sparse SPD Matrix, Make Sparse Uncorrelated, Make SPD Matrix, Make Swiss Roll.

% *** REFERENCES ***

% *** ACKNOWLEDGMENT ***
\section*{Acknowledgment}
The authors would like to thank the anonymous reviewers for their valuable comments and suggestions that helped improve the quality of this manuscript.

% *** DATA AVAILABILITY ***
\section*{Data Availability Statement}
All experiments were conducted using publicly available datasets that do not contain personal or sensitive information and are not subject to privacy restrictions. The datasets used in this study are documented in Section~\ref{sec:methodology}. No human subjects or animal experimentation was involved in this research. The computational experiments comply with ethical guidelines for responsible AI research. Code and additional materials are available from the corresponding author upon reasonable request.

% *** CONFLICT OF INTEREST ***
\section*{Conflict of Interest}
The authors declare that they have no conflicts of interest.

\end{document}